\documentclass[a4paper,12pt]{article}

\usepackage{a4wide}

\usepackage{float} 
\usepackage[utf8]{inputenc}
\usepackage[unicode=true,hypertexnames=false]{hyperref}
\usepackage{amsmath,amssymb,amsthm,amsfonts}
\usepackage{dsfont} % for \mathds font
\usepackage{mathtools} % for \coloneqq command
\usepackage[linesnumbered,ruled]{algorithm2e}
\let\oldnl\nl% Store \nl in \oldnl
\newcommand{\nonl}{\renewcommand{\nl}{\let\nl\oldnl}}% Remove line number for one line
\usepackage{xspace}
\usepackage[dvipsnames,table]{xcolor}
\usepackage{tikz}
\usetikzlibrary{decorations.pathreplacing,calc,tikzmark}
\usepackage{pgfplots}
\pgfplotsset{compat=newest}
\pgfplotscreateplotcyclelist{tikzcycle}{%
thick,blue,mark=square*\\ % q = 1/n^2
thick,red,mark=*\\ % q = ln(n)/n^2
thick,green!70!black,mark=triangle*\\ % q = 1/n
thick,orange,mark=diamond*\\ % q = 1
thick,blue,dashed,mark=square\\ % q = 1/n^2
thick,red,dashed,mark=o\\ % q = ln(n)/n^2
thick,green!70!black,dashed,mark=triangle\\ % q = 1/n
thick,orange,dashed,mark=diamond\\ % q = 1
}

\usepackage{multirow}
\usepackage{booktabs}
\usepackage{pdflscape} % to make a landscpe page with a table

\allowdisplaybreaks

\newcommand{\oea}{\mbox{$(1 + 1)$~EA}\xspace}

\newcommand{\onemax}{\textsc{OneMax}\xspace}
\newcommand{\leadingones}{\textsc{Leading\-Ones}\xspace}

\newcommand{\lo}{\textsc{LO}\xspace}

\newcommand{\N}{{\mathbb N}}
\newcommand{\R}{{\mathbb R}}

\DeclareMathOperator{\mut}{\texttt{mutate}}

\DeclareMathOperator{\Geom}{Geom}

\newtheorem{theorem}{Theorem}
\newtheorem{lemma}[theorem]{Lemma}

\let\originalleft\left
\let\originalright\right
\renewcommand{\left}{\mathopen{}\mathclose\bgroup\originalleft}
\renewcommand{\right}{\aftergroup\egroup\originalright}

\begin{document}

\title{Evolutionary Algorithms Are Significantly More Robust to Noise When They Ignore It}

\author{Denis Antipov \\
		Sorbonne Universit\'e,\\
		CNRS, LIP6 \\
        Paris, France \\
		\and
		Benjamin Doerr \\
		Laboratoire d'Informatique (LIX), \\
		CNRS, \'Ecole Polytechnique, \\
		Institut Polytechnique de Paris \\
		Palaiseau, France \\
}

\maketitle
{\sloppy
%comment out for arxiv
% \begin{large}

\begin{abstract}
	Randomized search heuristics (RSHs) are known to have a certain robustness to noise. Mathematical analyses trying to quantify rigorously how robust RSHs are to a noisy access to the objective function typically assume that each solution is re-evaluated whenever it is compared to others. This aims at preventing that a single noisy evaluation has a lasting negative effect, but is computationally expensive and requires the user to foresee that noise is present (as in a noise-free setting, one would never re-evaluate solutions).
	
    In this work, we conduct the first mathematical runtime analysis of an evolutionary algorithm solving a single-objective noisy problem without re-evaluations. We prove that the $(1+1)$ evolutionary algorithm without re-evaluations can optimize the classic LeadingOnes benchmark with up to constant noise rates, in sharp contrast to the version with re-evaluations, where only noise with rates $O(n^{-2} \log n)$ can be tolerated. 
    This result suggests that re-evaluations are much less needed than what was previously thought, and that they actually can be highly detrimental. The insights from our mathematical proofs indicate that this similar results are plausible for other classic benchmarks.
\end{abstract}

\section{Introduction}

In many real-world optimization problems, one does not have a perfect access to the problem instance, but, e.g., the objective function is mildly disturbed by \emph{noise}. Such noise can impose considerable difficulties to classic problem-specific algorithms. Randomized search heuristics, in contrast, are known to be able to cope with certain amounts of stochastic disturbances~\cite{BianchiDGG09,JinB05}. 

The ability to cope with noise has rigorously been studied and quantified via mathematical runtime analyses~\cite{NeumannW10,AugerD11,Jansen13,ZhouYQ19,DoerrN20}, that is, proven performance guarantees for certain algorithms in specific situations. For example, these results have shown that the $(1+1)$ evolutionary algorithm (EA) can solve the classic \onemax benchmark defined on bit-strings of length~$n$ in polynomial time when noise appears with a rate $O(n^{-1} \log n)$, but the runtime becomes super-polynomial for larger noise rates~\cite{Droste04,GiessenK16,Dang-NhuDDIN18}. For the equally popular benchmark \leadingones, only noise rate of order $O(n^{-2} \log n)$ admits a polynomial runtime~\cite{QianBYTY21,Sudholt21}. Evolutionary algorithms working with larger population sizes tend to be more robust to noise than the \oea, which is essentially a randomized hill-climber, see~\cite{GiessenK16,Sudholt21,AntipovDI24}.

These and almost all other mathematical runtime analyses of randomized search heuristics in the presence of noise assume that an already constructed solution is re-evaluated whenever it is compared to another solution. This seems to be justified by the fear that a noisy objective value, when not corrected via a renewed evaluation, could harm the optimization process for a long time. Interestingly, the only rigorous support for this fear are two analyses of an ant-colony optimizer in the presence of noise~\cite{DoerrHK12ants,SudholtT12}, one showing that this algorithm essentially cannot solve stochastic shortest paths problems (without re-evaluations) and the other proving a strong robustness to such disturbances when using re-evaluations. The only other runtime analysis not using re-evaluations is the recent work~\cite{DinotDHW23}. Since this work regards a multi-objective optimization problem and attributed the robustness to noise to the implicit diversity mechanisms of the multi-objective evolutionary algorithm regarded, it is hard to predict to what extent the findings generalize to the more classic case of single-objective optimization.

Re-evaluating each solution whenever its objective value is used by the RSH has two disadvantages. The obvious one is the increased computational cost. We note here that usually in black-box optimization, the function evaluations are the computationally most expensive part of the optimization process, up to the point that often the number of function evaluations is used as performance measure. A second problem with assuming re-evaluations in noisy optimization is that this requires the algorithm users to decide beforehand whether they expect to be prone to noise or not -- clearly, in a noise-free setting, one would not evaluate a solution more than once. 

\emph{Our contribution:} Given the apparent disadvantages of re-evaluations and the low theoretical support for these, in this work we study how a simple randomized search heuristic optimizes a classic benchmark problem when not assuming that solutions are re-evaluated. We analyze the setting best-understood in the case of re-evaluations, namely how the \oea optimizes the \leadingones benchmark, for which~\cite{Sudholt21} conducted a very precise runtime analysis, showing among others that a polynomial runtime can only be obtained for noise rates of at most $O(n^{-2} \log n)$. To our surprise, we obtain a much higher robustness to noise when not re-evaluating solutions. We prove that with noise rates up to a constant (which depends on the precise noise model, see Theorems~\ref{thm:all-one-bit} and~\ref{thm:all-bitwise} for the details), the \oea without re-evaluations optimizes the \leadingones benchmark in time quadratic in the problem size~$n$, which is the same asymptotic runtime as in the noise-free setting. This result suggests that the previous strong preference for re-evaluations is not as justified as the literature suggests. 

A closer inspection of our proofs also gives some insights in why working with possibly noisy objective values is less detrimental than previously thought, and sometimes even preferable. Very roughly speaking, we observe that noisy function values can also be overcome by generating a solution with true objective value at least as good as the previous noisy function value. Under reasonable assumptions (standard operators and standard noise models with not excessive noise rates), the mutation operator of the evolutionary algorithm has a higher variance than the noise, and consequently, it is easier to correct a noisy objective value by generating a sufficiently good solution than obtaining more noisy objective values due to new noise. From these insights, we are generally optimistic that our findings are not specific to the particular algorithm and benchmark studied in this work. 

Our theoretical study is supported by experimental results, which demonstrate that the \oea without re-evaluations has a performance similar to the noiseless setting, even when the noise rates are relatively high. In contrast, the re-evaluation approach makes the algorithm struggle at making a significant progress already at small problem sizes, when the noise is not too weak.

\section{Preliminaries}
\label{sec:prelims}

In this section we define the setting we consider, the notation and also mathematical tools we use in our analysis.
In this paper for any pair of integer numbers $a, b$ ($b \ge a$) by $[a..b]$ we denote an integer interval, that is, a set of all integer numbers which are at least $a$ and at most $b$. If $b < a$, then this denotes an empty set. By $\N$ we denote the set of all strictly positive integer numbers.

\subsection{\leadingones and Prior Noise}

\leadingones (\lo for brevity) is a benchmark function first proposed in~\cite{Rudolph97}, which is defined on bit strings of length $n$ (we call $n$ the \emph{problem size}) and which returns the size of the largest prefix of its argument consisting only of one-bits. More formally, for any bit string $x$ we have
\begin{align*}
	\leadingones(x) = \lo(x) = \sum_{i = 1}^n \prod_{j = 1}^i x_j.
\end{align*}

In this paper we consider optimization of \leadingones under prior noise. This means that each time we evaluate the \lo value of some bit string $x$, this bit string is first affected by some stochastic operator $N$. We call this operator \emph{noise}. Hence, instead of receiving the true value $\lo(x)$ we get $\lo(N(x))$. We consider the following two noise models.

\textbf{One-bit noise.} In this noise model with probability $q$ (which is called the \emph{noise rate}) operator $N(x)$ returns a bit string which is different from $x$ in exactly one bit, which is chosen uniformly at random. With probability $1 - q$ operator $N(x)$ returns an exact copy of $x$.

\textbf{Bitwise noise.} In this noise model with \emph{noise rate} $\frac{q}{n}$, operator $N(x)$ flips each bit in $x$ with probability $\frac{q}{n}$ independently from other bits, and returns the resulting bit string. We note that the definition of noise rate is different from the one-bit noise, however in both models the expected number of bits flipped by the noise is equal to $q$. Also, bitwise noise occurs with probability $1 - (1 - \frac{q}{n})^n = \Theta(\min(1, q))$, see, e.g., eq.~(1) in~\cite{Sudholt21}.

\subsection{The \oea}
\label{sec:oea}

\begin{algorithm}[t!]%
	\caption{The \oea maximizing a function $f:\{0,1\}^n \rightarrow \R$ under noise defined by operator $N$.}
	\label{alg:pseudo}
		\tcp{Initialization} 
		Sample $x \in \{0,1\}^{n}$ uniformly at random\;
		$f_x \gets f(N(x))$\;
		\tcp{Optimization}
		\While{not stopped}{
			$y \gets \mut(x)$\;
			$f_y \gets f(N(y))$\;
			\If{$f_y \ge f_x$}{
				$x \gets y$\;
				$f_x \gets f_y$\;
			} %if
		} %while 
\end{algorithm}

We consider a simple elitist evolutionary algorithm called the \oea. This algorithm stores one individual $x$ (we call it the \emph{parent} individual), which is initialized with a random bit string. Then until some stopping criterion is met\footnote{Similar to many other theoretical studies, we do not define the stopping criterion, but we assume that it does not stop before it finds an optimal solution.} it performs iterations, and in each iteration it creates offspring $y$ by applying a mutation operator to $x$. If the value of the optimized function on $y$ is not worse than its value on $x$, then $y$ replaces $x$ as the parent for the next iteration. Otherwise $x$ stays as the parent individual. The optimized function is called the \emph{fitness function} and its value on any individual $x$ is called the \emph{fitness} of $x$. In the rest of the paper we assume that function $f$ optimized by the \oea is \leadingones.

We consider two mutation operators, which in some sense similar to the two noise models. \textbf{One-bit mutation} flips exactly one bit chosen uniformly at random. \textbf{Standard bit mutation} flips each bit independently from other bits with probability $\frac{\chi}{n}$, where $\chi$ is a parameter of the mutation. We call $\frac{\chi}{n}$ the \emph{mutation rate}. 

Previous theoretical analyses of the \oea in noisy environments assumed that the fitness of the parent is re-evaluated in each iteration when it is compared with its offspring. In particular, Sudholt showed in~\cite{Sudholt21} that the \oea optimizes \leadingones in $\Theta(n^2) \cdot e^{\Theta(\min(n, pn^2))}$, where $p < \frac{1}{2}$ is the probability that prior noise occurs. Since for one-bit noise we have $p = q$ and for bitwise noise we have $p = \Theta(\min(1, q))$, this implies that for both noise models with $q = \omega(\frac{\log(n)}{n^2})$, the runtime of the \oea is super-polynomial.

Since that result by Sudholt indicated that the \oea with re-evaluations is not robust to even very small noise rates, we are interested in the behavior of the algorithm when it always uses the first evaluated value for the parent $x$ until this parent is not replaced with a new individual. This approach, in some sense, is similar to using the \oea on a noisy function without being aware that it is noisy (or just ignoring this fact). The pseudocode of the \oea using this approach is shown in Algorithm~\ref{alg:pseudo}.

We enumerate iterations of the \oea starting from zero, and for all $t \in \N \cup \{0\}$ we use the following notation to describe iteration $t$. By $x_t$ we denote the parent $x$ at the beginning of iteration $t$. Slightly abusing the notation, we write $\tilde f(x_t)$ to denote the fitness value $f_x$ stored in the algorithm at the beginning of iteration $t$. By $y_t$ and $\tilde f(y_t)$ we denote the offspring $y$ created in iteration $t$ and its noisy fitness $f_y$ correspondingly.

\subsection{Auxiliary Tools}

In this section we collect mathematical tools which help us in our analysis. We start with the following drift theorem, which is often used in runtime analysis of RSH to estimate the first hitting time of stochastic processes.

\begin{theorem}[Additive Drift Theorem~\cite{HeY04}, upper bound]
\label{thm:additive-drift}
    Let $(X_t)_{t \ge 0}$ be a sequence of non-negative random variables with a finite state space $S \subseteq \R_0^+$ such that $0 \in S$. Let $T \coloneqq \inf\{t \ge 0 \mid X_t = 0\}$. If there exists $\delta > 0$ such that for all $s \in S \setminus \{0\}$ and for all $t \ge 0$ we have $E[X_t - X_{t + 1} \mid X_t = s] \ge \delta$, then $E[T] \le \frac{E[X_0]}{\delta}$.
\end{theorem}

We also use the following inequality, which is a simplified version of Wald's equation shown in~\cite{DoerrK15}.

\begin{lemma}[Lemma~7 in~\cite{DoerrK15}]
    \label{lem:wald}
    Let $T$ be a random variable with bounded expectation and let $X_1, X_2, \dots$ be non-negative random variables with $E[X_i \mid T \ge i] \le C$ for some $C$ and for all $i \in \N$. Then
    \begin{align*}
        E\left[\sum_{i = 1}^T X_i\right] \le E[T] \cdot C.
    \end{align*}
\end{lemma}

For any individual $x$ evaluated by the \oea we call the \emph{active prefix} of $x$ the set of its first $\tilde f(x)$ bits, that is, the bits which ``pretended'' to be ones when we evaluate the fitness of $x$. The following lemma is an important ingredient of all our proofs.

\begin{lemma}
	\label{lem:active-prefix}
	For any individual $x$ and any $i \in [1..\tilde f(x)]$ the probability that there are exactly $i$ zero-bits in the active prefix of $x$ is at most the probability that the noise flipped $i$ particular bits when the fitness of $x$ was evaluated. In particular,
	\begin{enumerate}
		\item[(1)] for the one-bit noise with rate $q$ this probability is at most $\frac{q}{n}$ for $i = 1$ and zero for all other $i$,
		\item[(2)] for the bitwise noise with rate $\frac{q}{n}$ this probability is at most $(\frac{q}{n})^i$. 
	\end{enumerate}
\end{lemma}

\begin{proof}
	By the definition of the active prefix, when we evaluated the fitness of $x$, the noise affected it in such way that all its bits in the active prefix became one-bits. Hence if there are $i$ zero-bits in the active prefix of $x$, all of them have been flipped by the noise before the evaluation. Thus, flipping those $i$ bits is a super-event of the event when we have exactly $i$ zero-bits in the active prefix of $x$.

	For the one-bit noise the probability to flip a particular bit is $\frac{q}{n}$, and it cannot flip more than one bit. For the bitwise noise the probability that it flips $i$ particular bits is $(\frac{q}{n})^i$.
\end{proof}

The next two results are just short mathematical tools, which we formulate as separate lemmas to simplify the arguments in our main proofs.

\begin{lemma}
	\label{lem:ith-power-diff}
	For any real values $a$ and $b$ such that $a > b > 0$ and for any positive integer $i$ we have
	\begin{align*}
		\frac{a^{i + 1} - b^{i + 1}}{a^i - b^i} \le a + b.
	\end{align*}
\end{lemma}
\begin{proof}
	We have
	\begin{align*}
		(a + b)\left(a^i - b^i\right) &= a^{i + 1} - b^{i + 1} + a^i b - a b^i \ge a^{i + 1} - b^{i + 1}.
	\end{align*}
	Dividing both sides by the non-negative term $(a^i - b^i)$, we obtain the lemma statement.
\end{proof}

\begin{lemma}
    \label{lem:sum-integral}
    For all $a > 1$ and all integer $n \ge 2$ we have
    \begin{align*}
        \sum_{j = 2}^n \frac{1}{a^j - 1} \le \frac{\ln\left(1 - \frac{1}{a}\right)}{\ln(a)}.
    \end{align*}
\end{lemma}
\begin{proof}
    	Since $a > 1$, the function $g(x) = \frac{1}{a^x - 1}$ is monotonically decreasing in $x$ in interval $(0, +\infty)$. Thus we can bound the sum above by a corresponding integral. We then obtain
	\begin{align*}
		\sum_{j = 2}^n \frac{1}{a^j - 1} &\le \int_1^n \frac{dx}{a^x - 1} = \int_1^n \frac{a^{-x}dx}{1 - a^{-x}} = \frac{1}{\ln a} \int_1^n \frac{d\left( 1 - a^{-x}\right)}{1 - a^{-x}} \\
        &= \frac{1}{\ln a} \ln\left(1 - a^{-x}\right) \bigg|_1^n = \frac{\ln\left(1 - a^{-n}\right)}{\ln a} - \frac{\ln\left(1 - a^{-1}\right)}{\ln a} \le - \frac{\ln\left(1 - \frac{1}{a}\right)}{\ln a}, 
	\end{align*}
 where the last step is justified by $\frac{\ln(1 - a^{-n})}{\ln a} < 0$.
\end{proof}

We also use the following inequality which follows from Inequality~3.6.2 in~\cite{vasic2012analytic}.

\begin{lemma}
\label{lem:e-estimate}
    For any $n > 0$ and any $x \in [0, n]$ we have 
    \begin{align*}
        e^x - \frac{x^2}{2n} \le \left(1 - \frac{x}{n}\right)^n \le e^x.
    \end{align*}
\end{lemma}

\section{Runtime Analysis}

\subsection{One-bit Noise and One-bit Mutation}
\label{sec:all-one-bit}

We start our analysis with the most simple case that we have one-bit noise with rate $q \in (0, 1)$ and the \oea uses one-bit mutation. We aim to estimate the expected number of iterations it takes the algorithm to find the optimum of \leadingones when it does not re-evaluate the fitness of the parent solution. The proof idea of this situation will later be used for all other combinations of mutation and noise. The main result of this section is the following theorem.

\begin{theorem}
	\label{thm:all-one-bit}
	Consider a run of the \oea with one-bit mutation optimizing \leadingones under one-bit noise which occurs with probability $q < 1$. Then the EA finds the optimum (the all-ones bit string) and evaluates it properly in expected number of at most $\frac{(1 + q)n^2}{(1 - q)^2} + \frac{3q}{2(1 - q)}$ iterations.
\end{theorem}

Before we prove Theorem~\ref{thm:all-one-bit}, we need some preparation steps. At the start of any iteration $t \in \N$ the algorithm can be in one of three states: with $f(x_t) = \tilde f(x_t)$ (we call this state $S_=$), with $f(x_t) > \tilde f(x_t)$ (state $S_>$), or with $f(x_t) < \tilde f(x_t)$ (state $S_<$). We divide a run of the \oea into phases, and each phase (except, probably, the first one) starts in state $S_=$ and ends in the next iteration after which the \oea starts also in $S_=$. 

More formally, phases are defined as follows. Let $s_t$ be the state of the algorithm in iteration $t$ for all $t = 0, 1, \dots$, and let $Q$ be the set of all iterations, in which the algorithm is in state $S_=$, that is, $Q = \{t \mid s_t = S_=\}$. For all $i \in \N$ let $\tau_i$ be the $i$-th element of $Q$ enumerating them in ascending order. Then for all $i \in \N$ \emph{phase $i$} is defined as the integer interval $[\tau_i..\tau_{i + 1} - 1]$. Phase $0$ is defined as $[0..\tau_1 - 1]$. For all $i$ by the \emph{length of phase $i$} we denote its cardinality, which is $\tau_{i + 1} - \tau_i$ for phases $i > 1$ and which is $\tau_1$ for $i = 0$.

The following lemma estimates the expected length of one phase.

\begin{lemma}
	\label{lem:phase-one-bit}
	Consider a run of the \oea with one-bit mutation on \leadingones under one-bit noise with rate $q < 1$. For all $i \ge 1$ the expected length of phase $i$ is $E[\tau_{i + 1} - \tau_i] \le \frac{1 + q}{1 - q}$. The expected length of phase $0$ is $E[\tau_1] \le \frac{3q}{2(1 - q)}$.
\end{lemma}

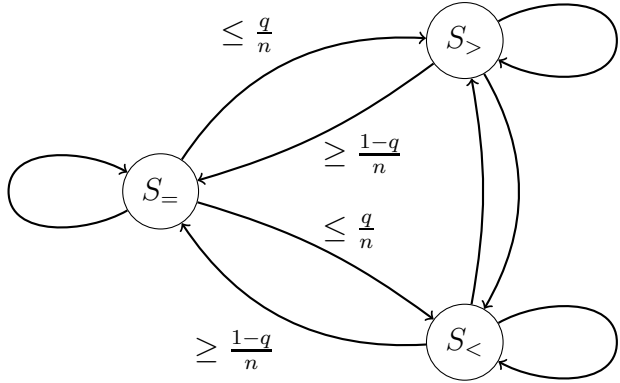
\begin{figure}
	\begin{center}
		\begin{tikzpicture}
			\node (seq) [draw, circle] at (0, 0) {$S_=$};
			\node (sg) [draw, circle] at (4, 2) {$S_>$};
			\node (sl) [draw, circle] at (4, -2) {$S_<$};

			\draw [->, thick] (seq) to[out=210, in=270] (-2, 0) to[out=90, in=150] (seq);
			\draw [->, thick] (sl) to[out=30, in=90] (6, -2) to[out=270, in=330] (sl);
			\draw [->, thick] (sg) to[out=30, in=90] (6, 2) to[out=270, in=330] (sg);

			\draw [->, thick] (seq) to[bend left] node [above left, midway] {$\le \frac{q}{n}$} (sg);
			\draw [->, thick] (sg) to[bend left=10pt]node [below right=-4pt, midway] {$\ge \frac{1 - q}{n}$} (seq);
			\draw [->, thick] (seq) to[bend left=10pt] node [above right=-4pt, midway] {$\le \frac{q}{n}$} (sl);
			\draw [->, thick] (sl) to[bend left]node [below left, midway] {$\ge \frac{1 - q}{n}$} (seq);

			\draw [->, thick] (sg) to[bend left=30pt] (sl);
			\draw [->, thick] (sl) to[bend right=10pt] (sg);
			
		\end{tikzpicture}
	\end{center}
	\caption{The Markov Chain used in the proof of Lemma~\ref{lem:phase-one-bit} and the possible transitions between the states.}
	\label{fig:one-bit-illustation}
\end{figure}

\begin{proof}
	We describe the algorithm states as a Markov chain shown in Figure~\ref{fig:one-bit-illustation} and find the transition probabilities between the states.

	When the algorithm starts some iteration $t$ in state $S_=$, then to go to state $S_<$ in one iteration, the noise must increase the fitness of the offspring $y_t$ (and also this noisy fitness must be not worse than $\tilde f(x_t)$). This means that $y_t$ must have at least one zero-bit in its active prefix. By Lemma~\ref{lem:active-prefix}, the probability of this event is at most $\frac{q}{n}$. To go to state $S_>$ from $S_=$, we need to have $f(y_t) > \tilde f(y_t)$ and also $\tilde f(y_t) \ge \tilde f(x_t) = f(x_t)$. Hence, the algorithm must create an offspring $y_t$ with strictly better \lo value than the one of the parent and then noise must occur and reduce its fitness. This is possible only if the mutation flips the first zero-bit, which in the case of one-bit mutation has probability of $\frac{1}{n}$, and then noise occurs with probability $q$. Consequently, the transition probabilities from state $S_=$ to states $S_<$ and $S_>$ are both at most $\frac{q}{n}$.

	When the algorithm starts an iteration $t$ in state $S_>$, then to move to state $S_=$ the algorithm can flip any of the bits in positions $[f(x_t)..n]$ (then the fitness of $y_t$ will be at most by one smaller than the true fitness of $x_t$, and therefore it is not worse than the noisy fitness $\tilde f(x_t)$) and then be lucky to have no noise. The probability of this is at least $\frac{n - f(x_t) + 1}{n} \cdot (1 - q) \ge \frac{1 - q}{n}$.
	
	When the algorithm starts iteration $t$ in state $S_<$, then $x_t$ has exactly one zero-bit in its active prefix, since in one-bit noise model the real offspring differs from the noisy one in at most one bit. To go to state $S_=$, the \oea can flip the only zero-bit in the active prefix of $x_t$ via mutation (thus, get $y_t$ with $f(y_t) = \tilde f(x_t)$) and then have no noise (thus, have $\tilde f(y_t) = f(y_t)$). The probability of this event is $\frac{1}{n} \cdot (1 - q) = \frac{1 - q}{n}$.
	
	With these transition probabilities, we can estimate the expected number of steps it takes to go to state $S_=$ starting from each state. Let $T_=$, $T_>$ and $T_<$ be the time (number of iterations) until the \oea reaches state $S_=$ starting from states $S_=$, $S_>$ and $S_<$ respectively. Since transition probabilities from $S_>$ and $S_<$ to $S_=$ are both at least $\frac{1 - q}{n}$, then both $T_>$ and $T_<$ are dominated by a geometric distribution $\Geom(\frac{1 - q}{n})$, and their expected values are at most $\frac{n}{1 - q}$. To bound $E[T_=]$, we estimate 
	\begin{align*}
		E[T_=] &=   1 + \Pr[S_= \to S_<] E[T_<] +  \Pr[S_= \to S_>] E[T_>] \\
			   &\le 1 + \left(\frac{q}{n} + \frac{q}{n}\right) \cdot \frac{n}{1 - q} = 1 + \frac{2q}{1 - q} = \frac{1 + q}{1 - q}.
	\end{align*}
	Noting that $\tau_{i + 1} - \tau_i = T_=$ for all $i \ge 1$ completes the proof for all phases, except phase 0. 

    In iteration $t = 0$ the algorithm is in state $S_<$ only if the noise flipped the first zero-bit in the initial individual. The probability of this event is $\frac{q}{n}$. The algorithm is in state $S_>$ in iteration $t = 0$, if noise flipped one of the one-bits in the prefix of the initial individual. For all $k \in [0..n]$ the probability that the \lo value of the initial individual is $k$ is $2^{-\min(k + 1, n)}$ (since it requires a particular value for the first $\min(k + 1, n)$ bits). If the \lo value is $k$, then the probability that noise flips one of the $k$ one-bits in the prefix is $\frac{qk}{n}$. By the law of total probability and by the well-known estimate $\sum_{k = 1}^{+\infty} kx^{k - 1} = \frac{x}{(1 - x)^2}$ which holds for all $x \in (0, 1)$, the probability of starting in state $S_>$ is then
    \begin{align*}
        \sum_{k = 0}^{n - 1} 2^{-(k + 1)} \cdot \frac{qk}{n} + q2^{-n} &= \frac{q}{n} \left(\sum_{k = 1}^{n - 1} k2^{-(k + 1)} + n2^{-n}\right) = \frac{q}{n} \left(\sum_{k = 1}^{n - 1} k2^{-(k + 1)} + n\sum_{k = n}^{+\infty} 2^{-(k + 1)}\right) \\
        &\le \frac{q}{4n} \sum_{k = 1}^{+\infty} k2^{-(k - 1)} = \frac{q}{4n} \cdot 2 = \frac{q}{2n}. 
    \end{align*}

    By the law of total expectation, we have
    \begin{align*}
        E[\tau_1] &= \Pr[s_0 = S_=] \cdot 0 + \Pr[s_0 = S_<] \cdot E[T_<] + \Pr[s_0 = S_>] \cdot E[T_>] \\  
        &\le \left(\frac{q}{n} + \frac{q}{2n}\right) \cdot \frac{n}{1 - q} = \frac{3q}{2(1 - q)}.
    \end{align*}
\end{proof}

With the estimate of the expected time of one phase, we can prove Theorem~\ref{thm:all-one-bit}.

\begin{proof}[Proof of Theorem~\ref{thm:all-one-bit}]
    We first define a \emph{super-phase} of the algorithm as follows. Let $R = \{\tau_{i + 1} \mid f(x_{\tau_{i + 1}}) > f(x_{\tau_i}), i \in \N\}$ that is, $R$ is a set of iterations which start a new phase such that the new phase starts with a strictly higher fitness than the previous phase (in terms of both true and noisy fitness, since they are equal in the beginning of any phase, except phase $0$). Note that $R$ has at most $n$ elements, since there are $n$ different fitness values. Let $t_0 = \tau_1$ and for all $i \in [1..|R|]$ let $t_i$ be the $i$-th element of $R$, if we sort them in ascending order. Then we define the $i$-th super-phase as interval $[t_i..t_{i + 1}]$ for all $i \in [0..|R| - 1]$.

    Consider some particular, but arbitrary super-phase $i$. It consists of one or more phases, and we denote the length of $j$-th phase in this super-phase by $T_j$. We call a phase \emph{successful}, if the next phase starts with a strictly better fitness than this phase. This implies that a super-phase ends after a successful phase occurs. A phase is successful, if it consists of one iteration in which mutation flips the first zero-bit of $x$ and noise does not occur. The probability of this event is $\frac{1 - q}{n}$. Therefore, the number of phases $N$ in each super-phase is dominated by a geometric distribution $\Geom(\frac{1 - q}{n})$. By Lemma~\ref{lem:wald} and by the estimate of the expected length of a phase from Lemma~\ref{lem:phase-one-bit}, we have that the expected length of any super-phase is

    \begin{align*}
        E[t_{i + 1} - t_i] = \sum_{j = 1}^N E[T_j] \le E[N] \cdot \frac{1 + q}{1 - q} \le \frac{(1 + q)n}{(1 - q)^2}.
    \end{align*}

    The total runtime consists of the length of phase $0$ and the sum of length of all super-phases. Recalling that the number of super-phases $|R|$ is at most $n$, by Lemma~\ref{lem:wald} we obtain that the total runtime is 
    \begin{align*}
        E[T] = E[\tau_1] + \sum_{k = 1}^{|R| - 1} E[t_{i + 1} - t_i] \le \frac{3q}{2(1 - q)} + n \cdot \frac{(1 + q)n}{(1 - q)^2} = \frac{(1 + q)n^2}{(1 - q)^2} + \frac{3q}{2(1 - q)}.
    \end{align*}
\end{proof}

\subsection{Bitwise Noise and Bitwise Mutation} 
\label{sec:all-bitwise}

In this section we study the case when the \oea uses standard mutation and noise is bitwise. This implies that noise has a non-zero probability to flip any number of bits, and thus there might be more than one zero-bit in the active prefix of the current individual $x$. However, as we show in this section, if noise is not too strong, then the \oea can handle situations with $k \ge 1$ zero-bits in the active prefix in time of order $O(\frac{1}{p})$, where $p$ is the probability that such situation occurs. This implies that the unfortunate events when the parent has too many zero-bits in its active prefix only add at most a constant factor to the runtime compared to the noiseless setting. The main result of this section is the following theorem.

\begin{theorem}
    \label{thm:all-bitwise}
	Consider a run of the \oea with standard bit mutation with rate $\frac{\chi}{n}$ optimizing \leadingones under bitwise noise with rate $\frac{q}{n}$. Let $r \coloneqq \chi + q - \frac{2\chi q}{n}$ and assume that (i) $\chi = \Theta(1)$ and $q = O(1)$ and (ii) there exists a constant $c \in (0, 1)$ such that $- \frac{\ln(1 - \frac{q}{r})}{\ln\frac{r}{q}} \le (1 - c)e^{-(\chi + q)}$.
    Then the expected number of iterations until the \oea finds the optimum (the all-ones bit string) and evaluates it properly is at most
    \begin{align*}
        \frac{n^2}{c} \left(\frac{e^{\chi + q}}{\chi} + q\left(\frac{e^{\chi + q}}{\chi}\right)^2 \right) + O(n).
    \end{align*}
\end{theorem}

Before we start the proof, we discuss condition (ii) on $\chi$ and $q$ (and $r$, which is a function of $\chi$ and $q$). If $\chi = \Theta(1)$ and $q = o(1)$, then the left part $-\frac{\ln(1 - \frac{q}{r})}{\ln\frac{r}{q}}$ is $o(1)$, while $e^{-(\chi + q)} = \Omega(1)$. Hence, in this case we can choose $c$ close to one, namely $c = 1 - o(1)$. The upper bound on the runtime is then $e^{\chi}\chi^{-1} n^2 (1 + o(1))$. For larger $q = \Theta(1)$ this condition can be satisfied only for some range of $\chi$. If we express $q$ as a fraction of $\chi$, that is, $q = \alpha \chi$ for some $\alpha > 0$, then condition (ii) can be rewritten as 
\begin{align*}
    \frac{\ln(1 + \alpha)e^{(\alpha + 1)\chi}}{\ln(1 + \alpha)  - \ln \alpha} \le 1 - c - o(1).
\end{align*}
From this inequality it trivially follows that to have the left part less than one, we need $\ln \alpha$ to be negative, that is, we need the noise rate to be smaller than the mutation rate. This relation between those two rates ensures that if we get a parent $x$ with wrongly evaluated fitness, then variation of the mutation operator must be stronger than variation of noise, so that fixing this faulty situation was more likely than making it worse. More precise computation\footnote{The code used to perform these computations can be found in the supplementary material, the file name is ``calculate\_alpha.py''.} allows to see that $\chi \approx 1.4$ allows the left part to be less than one for the maximum possible $q \approx 0.39$. We note, however, that this is likely not a tight bound on the maximum noise rate which can be tolerated by the \oea, and in our experimental investigation in Sections~\ref{sec:experiments} we show that the algorithm can efficiently optimize \leadingones even when $q = 1$. 

To prove Theorem~\ref{thm:all-bitwise}, we first need several auxiliary results. The next lemma will help to estimate the probability of a ``good event'', when the algorithm reduces the number of zero-bits in the active prefix of $x$.

\begin{lemma}
	\label{lem:mut-and-noise}
	Assume that both $q$ and $\chi$ are $O(1)$ and let $r = \chi + q - \frac{2q\chi}{n}$ (as in Theorem~\ref{thm:all-bitwise}). Let $x$ be some arbitrary bit string. Let $\tilde y$ be an offspring of $x$ that was obtained by standard bit mutation with rate $\frac{\chi}{n}$ to $x$ and then bitwise noise with rate $\frac{q}{n}$. Let also $S$ be an arbitrary non-empty subset of $[1..n]$ and let $i$ be its size $|S|$. Consider the event that these three conditions are satisfied:
	\begin{enumerate}
		\item[(1)] each bit with position in $S$ was flipped by exactly one of mutation or noise; 
		\item[(2)] at least one bit with position in $S$ was flipped by mutation; and
		\item[(3)] all bits with positions not in $S$ have not been flipped.
	\end{enumerate}
	The probability of this event is at least
	\begin{align*}
		\left(\frac{r^i - q^i}{n^i}\right) \left(e^{-(\chi + q)} - O\left(\frac{1}{n}\right)\right).
	\end{align*}
\end{lemma}
\begin{proof}
	Consider one particular, but arbitrary bit. The probability that it is flipped by either mutation or noise, but not by both of them is
	\begin{align*}
		\frac{\chi}{n}\left(1 - \frac{q}{n}\right) + \frac{q}{n}\left(1 - \frac{\chi}{n}\right) = \frac{\chi + q - \frac{2q\chi}{n}}{n} = \frac{r}{n}.
	\end{align*}
	The probability that it happens with all $i$ bits in positions in $S$ is then $(\frac{r}{n})^i$.
	None of those $i$ bits are flipped by mutation only in the case of the sub-event, when all of those bits were flipped by noise. The probability of this sub-event is $(\frac{q}{n})^i$.
	Hence, the first two conditions are satisfied with probability exactly $(\frac{r^i - q^i}{n^i})$.
	Each bit outside of $S$ is not flipped by mutation nor by noise with probability 
	\begin{align*}
		\left(1 - \frac{\chi}{n}\right) \left(1 - \frac{q}{n}\right) = 1 - \frac{\chi + q}{n} + \frac{q\chi}{n^2}.
	\end{align*}
	The probability that all bits with position not in $S$ are not flipped then is
	\begin{align*}
		\left(1 - \frac{\chi + q}{n} + \frac{q\chi}{n^2}\right)^{n - i} &\ge e^{-\left(\chi + q - \frac{q\chi}{n}\right)} - \frac{\left(\chi + q - \frac{q\chi}{n}\right)^2}{2n} \ge e^{-(\chi + q)} \left(1 + \frac{q\chi}{n}\right) - O\left(\frac{1}{n}\right) \\
		&= e^{-(\chi + q)} - O\left(\frac{1}{n}\right),
	\end{align*}
	where in the first step we used Lemma~\ref{lem:e-estimate}.
\end{proof}

Similar to the previous section, we distinguish different states in which the algorithm can occur. However, since bitwise noise can lead to any number of zero-bits in the active prefix, we need a more detailed description of state $S_<$ (when the true fitness $f(x_t)$ of the parent is smaller than the stored noisy fitness $\tilde f(x_t)$). We define states $S_=$ and $S_>$ similar to Section~\ref{sec:all-one-bit}. Namely, the algorithm is in state $S_=$ in iteration $t$, if $f(x_t) = \tilde f(x_t)$ and it is in state $S_>$, if $f(x_t) > \tilde f(x_t)$. For all $j \in [1..n]$ we also say that the algorithm is in state $S_j$ in iteration $t$, if $f(x_t) < \tilde f(x_t)$ and there are exactly $j$ zero-bits in the active prefix of $x_t$. We divide the run of the algorithm into phases in the similar way as we did in Section~\ref{sec:all-one-bit}. Namely, for all $i \in \N$ we define $\tau_i$ as the $i$-th iteration in which the algorithm is in state $S_=$ and we define phase $i$ as an integer interval $[\tau_i..\tau_{i + 1} - 1]$. Phase $0$ is defined as $[0..\tau_1 - 1]$. The length of a phase is its cardinality. The following lemma estimates an expected length of a phase, similar to Lemma~\ref{lem:phase-one-bit}.

\begin{lemma}
	\label{lem:phase-bitwise}
	Consider a run of the \oea with standard bit mutation with rate $\frac{\chi}{n}$ on \leadingones under bitwise noise with rate $\frac{q}{n}$. Let $\chi = \Theta(1)$ and $q = O(1)$. Let $r = \chi + q - \frac{2q\chi}{n}$ (that is, $\frac{r}{n}$ is the probability that a particular bit flipped by either mutation or noise, but not by both) and assume that there exists constant $c \in (0, 1)$ such that $- \frac{\ln(1 - \frac{q}{r})}{\ln\frac{r}{q}} \le (1 - c)e^{-(\chi + q)}$. Then for all $i \in \N$ we have 
	\begin{align*}
		E[\tau_{i + 1} - \tau_i] \le \frac{1}{c} \left(1 + \frac{qe^{\chi + q}}{\chi}\right) + O\left(\frac{1}{n}\right) = O(1).
	\end{align*}
    The expected length of phase 0 is
    \begin{align*}
        E[\tau_1] \le \frac{1}{c}\left((1 - c) + \frac{e^{\chi + q}}{2} + \frac{qe^{2(\chi + q)}}{\chi}\right) + O\left(\frac{1}{n}\right) = O(1).
    \end{align*}
\end{lemma}
\begin{proof}
	Since we are aiming at an asymptotic statement, we may assume that $n$ is sufficiently large. To prove this lemma, we use the additive drift theorem. For this reason we assign the following potential $\Phi$ to states $S_=$, $S_>$ and all $S_j$.
	\begin{align*}
		\Phi(s) = \begin{cases}
			\frac{n^j}{r^j - q^j}, &\text{ if } s = S_j \text{ for all } j \in [1..n], \\
			1 + \frac{qe^{\chi + q}}{r - q}, &\text{ if } s = S_>, \\
			0, &\text{ if } s = S_=,
		\end{cases}
	\end{align*}

	Consider the process $X_t = \Phi(s_t)$, where $s_t$ is the state of the algorithm in iteration $t$, and an arbitrary phase $i$ with $i > 0$ which starts at iteration $\tau_i$. Then $\tau_{i + 1}$ is the first iteration after $\tau_i$ where $X_{\tau_{i + 1}} = 0$. Note that $X_{\tau_i} = 0$, but in iteration $\tau_i + 1$ we can have larger potentials $X_{\tau_i + 1}$.
	If we find some $\delta > 0$ such that for all $t > \tau_i$ and for all possible values $\phi$ of the potential we have $E[X_{t + 1} - X_{t} \mid X_t = \phi] \ge \delta$, then by the additive drift theorem (Theorem~\ref{thm:additive-drift}) we obtain $E[\tau_{i + 1} - (\tau_i + 1)] \le \frac{E[X_{\tau_i + 1}]}{\delta}$. Hence, the expected length of the phase is at most $E[\tau_{i + 1} - \tau_i] \le 1 + \frac{E[X_{\tau_i + 1}]}{\delta}$.

	To estimate $E[X_{t + 1} - X_{t} \mid X_t = \phi]$, we compute the transition probabilities between different states. By Lemma~\ref{lem:active-prefix}, for all $j \in [1..\tilde f(x)]$ the probability to go from any state to state $S_j$ is at most $(\frac{q}{n})^j$, and for $j > \tilde f(x)$ this probability is zero.

	When the algorithm is in state $S_j$ for some $j \in [1..n]$, then to go to a state with a smaller potential (that is, to either $S_=$, $S_>$ or $S_k$ with $k < j$) in one iteration it is sufficient that in $y_t$ the mutation flips at least one zero-bit in the active prefix of $x_t$ (and does not flip any other bit in the active prefix except for other zero-bits) and then, when we evaluate $\tilde f(y_t)$, the noise flips all remaining zero-bits in the active prefix (but does not flip any other bit). The probability of this event can be estimated by applying Lemma~\ref{lem:mut-and-noise} with $S$ being the set of zero-bits in the active prefix of $x_t$, that is, this probability is at least 
	\begin{align*}
		\left(\frac{r^j - q^j}{n^j}\right) \left(e^{-(\chi + q)} - O\left(\frac{1}{n}\right)\right).
	\end{align*}
	If we go to a state with a smaller potential, then we reduce the potential by at least $\frac{n^j}{r^j - q^j} - \frac{n^{j - 1}}{r^{j -1} - q^{j - 1}}$, when $j > 1$, and by $\frac{n}{r - q} - 1 - \frac{qe^{\chi + q}}{r - q}$, when $j = 1$. Therefore, when the algorithm is in state $S_j$ with $j \ge 2$, then the drift of the potential is at least
	\begin{align}
		\label{eq:drift-i-ge-2}
		\begin{split}
			E[X_t &- X_{t + 1} \mid s_t = S_j, j \ge 2] \\
            \ge& \left(\frac{r^j - q^j}{n^j}\right) \left(e^{-(\chi + q)} - O\left(\frac{1}{n}\right)\right) \left(\frac{n^j}{r^j - q^j} - \frac{n^{j - 1}}{r^{j -1} - q^{j - 1}}\right) \\ 
			&- \sum_{k = j+1}^n \left(\frac{q}{n}\right)^k \frac{n^k}{r^k - q^k}.
		\end{split}
	\end{align}
	We now estimate the positive and the negative terms separately. For the positive part of the drift, by Lemma~\ref{lem:ith-power-diff}, we have 
	\begin{align}
		\label{eq:negative-drift-i-ge-2}
		\begin{split}
			\left(\frac{r^j - q^j}{n^j}\right) &\left(e^{-(\chi + q)} - O\left(\frac{1}{n}\right)\right) \left(\frac{n^j}{r^j - q^j} - \frac{n^{j - 1}}{r^{j - 1} - q^{j - 1}}\right) \\ 
			&= \left(e^{-(\chi + q)} - O\left(\frac{1}{n}\right)\right) \left(1 - \frac{r^j - q^j}{n(r^{j - 1} - q^{j - 1})}\right) \\
			&\ge \left(e^{-(\chi + q)} - O\left(\frac{1}{n}\right)\right)\left(1 - \frac{r + q}{n}\right) = e^{-(\chi + q)} - O\left(\frac{1}{n}\right),
		\end{split}
	\end{align}
	since we assume that $\chi = \Theta(1)$ and $q = O(1)$. For the negative term, we have
	\begin{align*}
		\sum_{k = j + 1}^n \left(\frac{q}{n}\right)^k \frac{n^k}{r^k - q^k} = \sum_{k = j + 1}^n \frac{q^k}{r^k - q^k} = \sum_{k = j + 1}^n \frac{1}{\left(\frac{r}{q}\right)^k - 1}.
	\end{align*}
	By Lemma~\ref{lem:sum-integral} with $a = \frac{r}{q} > 1$ and by lemma conditions, we have 
	\begin{align*}
		\sum_{k = j + 1}^n \frac{1}{\left(\frac{r}{q}\right)^k - 1} &\le - \frac{\ln\left(1 - \frac{q}{r}\right)}{\ln\frac{r}{q}} \le (1 - c)e^{-(\chi + q)}.
	\end{align*}
	Putting this and also eq.~\eqref{eq:negative-drift-i-ge-2} into eq.~\eqref{eq:drift-i-ge-2}, we obtain
	\begin{align*}
		E\left[X_t - X_{t + 1} \mid s_t = S_j, j \ge 2\right] \ge e^{-(\chi + q)} - O\left(\frac{1}{n}\right) - (1 - c)e^{-(\chi + q)} = ce^{-(\chi + q)} - O\left(\frac{1}{n}\right).
	\end{align*}

	For $S_1$ we can write a similar inequality as eq.~\eqref{eq:drift-i-ge-2}, but the positive term is
	\begin{align*}
		\left(\frac{r - q}{n}\right) &\left(e^{-(\chi + q)} - O\left(\frac{1}{n}\right)\right) \left(\frac{n}{r - q} - 1 - \frac{qe^{\chi + q}}{r - q}\right) \\ 
		&= \left(e^{-(\chi + q)} - O\left(\frac{1}{n}\right)\right) \left(1 - \frac{r - q + qe^{\chi + q}}{n}\right) \\
		&= e^{-(\chi + q)} - O\left(\frac{1}{n}\right),
	\end{align*}
	which is asymptotically the same bound as in eq.~\eqref{eq:drift-i-ge-2}. Hence, the bound
	\begin{align*}
		E\left[X_t - X_{t + 1} \mid s_t = S_j\right] \ge ce^{-(\chi + q)} - O\left(\frac{1}{n}\right)
	\end{align*}
	holds for all $j$.

	When the algorithm is in state $S_>$, then to get to state $S_=$ it is enough to flip no bits by mutation or noise. The resulting offspring then is a copy of its parent and its fitness is equal to its noisy fitness, which in $S_>$ is larger than the fitness stored by the algorithm (thus, the new individual replaces the parent). The probability of this event is
	\begin{align*}
		\left(1 - \frac{\chi}{n}\right)^n \left(1 - \frac{q}{n}\right)^n = e^{-(\chi + q)} - O\left(\frac{1}{n}\right).
	\end{align*}
	The transition probabilities from $S_>$ to the states $S_i$ can be computed as before. Hence, the drift in state $S_>$ satisfies
	\begin{align*}
		E[X_t - X_{t + 1} \mid s_t = S_>] &\ge \left(1 + \frac{qe^{\chi + q}}{r - q}\right)\left(e^{-(\chi + q)} - O\left(\frac{1}{n}\right)\right) - \sum_{i = 1}^n \left(\frac{q}{n}\right)^i \frac{n^i}{r^i - q^i} \\
		&\ge e^{-(\chi + q)} + \frac{q}{r - q} - O\left(\frac{1}{n}\right) - \frac{q}{r - q} - \sum_{i = 2}^n \frac{q^i}{r^i - q^i} \\
		&\ge e^{-(\chi + q)} -  (1 - c)e^{-(\chi + q)} - O\left(\frac{1}{n}\right) = ce^{-(\chi + q)} - O\left(\frac{1}{n}\right),
	\end{align*}
	where we used Lemma~\ref{lem:sum-integral} and lemma conditions to obtain the last line.
	 
	We have now shown that for every state $s \ne S_=$ the expected progress of $X_t$ is at least $\delta = ce^{-(\chi + q)} - O(\frac{1}{n})$. To apply the additive drift theorem, we now estimate $E[X_{\tau_i + 1}]$. For this we note that the probability to go from $S_=$ to any of $S_j$ in one iteration is at most $(\frac{q}{n})^j$ by Lemma~\ref{lem:active-prefix}. To go to state $S_>$ from $S_=$, the algorithm has to create an offspring with true fitness better than the fitness of the current parent. Hence, the probability of this event is at most $\frac{\chi}{n}$ (that is, the probability of flipping the first zero-bit of the parent with mutation). With these estimates of transition probabilities, we obtain
	\begin{align*}
		E[X_{\tau_i + 1}] &\le \left(1 + \frac{qe^{\chi + q}}{r - q}\right) \frac{\chi}{n} + \sum_{j = 1}^n \left(\frac{q}{n}\right)^j \frac{n^j}{r^j - q^j} \\
		&\le O\left(\frac{1}{n}\right) + \frac{q}{r - q} + (1 - c)e^{-(\chi + q)} = O(1).
	\end{align*}

	By the additive drift theorem (Theorem~\ref{thm:additive-drift}), we have
	\begin{align*}
		E[\tau_{i + 1} - \tau_i] &\le 1 + \frac{E[X_{\tau_i + 1}]}{\delta} \le 1 + \frac{\frac{q}{r - q} + (1 - c)e^{-(\chi + q)} + O\left(\frac{1}{n}\right)}{ce^{-(\chi + q)} - O\left(\frac{1}{n}\right)} \\
        &= \frac{\frac{q}{r - q} + e^{-(\chi + q)} + O\left(\frac{1}{n}\right)}{ce^{-(\chi + q)} - O\left(\frac{1}{n}\right)} = \frac{qe^{\chi + q}}{c(r - q)} + \frac{1}{c} + O\left(\frac{1}{n}\right) \\
        &= \frac{1}{c} \left(1 + \frac{qe^{\chi + q}}{\chi}\right) + O\left(\frac{1}{n}\right) =  O(1).
	\end{align*}

    For phase 0 we can use the additive drift theorem with the same potential. The estimates of the drift stay the same, but the expected initial potential $E[X_0]$ is different, and we compute an upper bound on it as follows. For all $j \in [0..n]$, to start in state $S_j$ the noise must flip the first $j$ zero-bits in the initial individual when the algorithm evaluates its fitness. By Lemma~\ref{lem:active-prefix}, the probability of this event is at most $(\frac{q}{n})^j$. To start in state $S_>$, the initial individual must have the true fitness at least one, that is, its first bit must be a one-bit. The probability of this event is at most $\frac{1}{2}$. Therefore, we have
    \begin{align*}
        E[X_0] &= \sum_{j = 1}^n \Pr[s_0 = S_j] \frac{n^j}{r^j - q^j} + \Pr[s_0 = S_>] \left(1 + \frac{qe^{\chi + q}}{r - q}\right) \\
        &\le (1 - c)e^{-(\chi + q)} + \frac{1}{2} \left(1 + \frac{qe^{\chi + q}}{r - q}\right).
    \end{align*}
    By the additive drift theorem (Theorem~\ref{thm:additive-drift}), the expected length of phase 0 is then
    \begin{align*}
        E[\tau_1] &\le \frac{E[X_0]}{\delta} \le \frac{(1 - c)e^{-(\chi + q)} + \frac{1}{2} \left(1 + \frac{qe^{\chi + q}}{r - q}\right)}{ce^{-(\chi + q)} - O\left(\frac{1}{n}\right)} \\
        &= \frac{1}{c}\left((1 - c) + \frac{e^{\chi + q}}{2} + \frac{qe^{2(\chi + q)}}{\chi}\right) + O\left(\frac{1}{n}\right).
    \end{align*}

\end{proof}

We are now in position to prove the main result of this section, Theorem~\ref{thm:all-bitwise}.

\begin{proof}[Proof of Theorem~\ref{thm:all-bitwise}]
    Since this proof repeats the proof of Theorem~\ref{thm:all-one-bit}, we omit most of the details, except the insignificant differences. We define \emph{super-phases} and \emph{successful phases} in the same way as in the proof of Theorem~\ref{thm:all-one-bit}. The probability that a phase is successful is at least the probability that in the first iteration of this phase mutation flips the first zero-bit of an individual and does not flip any other bit, and noise does not flip any bit. By Lemma~\ref{lem:e-estimate}, the probability of this event is at least
    \begin{align*}
        \frac{\chi}{n}\left(1 - \frac{\chi}{n}\right)^{n - 1}\left(1 - \frac{q}{n}\right)^{n} \ge \frac{\chi}{n}\left(e^{-\chi} - \frac{\chi^2}{2n}\right)\left(e^{-q} - \frac{q^2}{2n}\right) = \frac{\chi e^{-(\chi + q)}}{n} - O\left(\frac{1}{n^2}\right).
    \end{align*}
    Therefore, the number of phases $N$ in one super-phase is dominated by geometric distribution $\Geom(\frac{\chi e^{-(\chi + q)}}{n} - O(\frac{1}{n^2}))$. Hence, if we denote by $T_j$ the length of $j$-th phase in a super-phase, then by Lemmas~\ref{lem:wald} and~\ref{lem:phase-bitwise} the expected length of one super-phase is at most
    \begin{align*}
        E[t_{i + 1} - t_i] &= \sum_{j = 1}^N E[T_j] \le E[N] \left(\frac{1}{c} \left(1 + \frac{qe^{\chi + q}}{\chi}\right) + O\left(\frac{1}{n}\right)\right) \\
        &\le \frac{\frac{1}{c} \left(1 + \frac{qe^{\chi + q}}{\chi}\right) + O\left(\frac{1}{n}\right)}{\frac{\chi e^{-(\chi + q)}}{n} - O\left(\frac{1}{n^2}\right)} \\
        &= \frac{\frac{ne^{\chi + q}}{\chi} \cdot \frac{1}{c} \left(1 + \frac{qe^{\chi + q}}{\chi}\right) + O(1)}{1 - O\left(\frac{1}{n}\right)} \\
        &= \frac{n}{c} \left(\frac{e^{\chi + q}}{\chi} + q\left(\frac{e^{\chi + q}}{\chi}\right)^2 \right) + O(1).
    \end{align*}
    The optimum is reached after at most $n$ super-phases. If we also take into account phase $0$, then the total runtime is then at most
    \begin{align*}
        E[T] &= E[\tau_0] + \sum_{k = 1}^{|R| - 1} E[t_{i + 1} - t_i] \\
        &\le O(1) + n \cdot \left(\frac{n}{c} \left(\frac{e^{\chi + q}}{\chi} + q\left(\frac{e^{\chi + q}}{\chi}\right)^2 \right) + O(1)\right) \\
        &= \frac{n^2}{c} \left(\frac{e^{\chi + q}}{\chi} + q\left(\frac{e^{\chi + q}}{\chi}\right)^2 \right) + O(n).
    \end{align*}
\end{proof}

\subsection{Discussion of the Theoretical Results}

When the noise rate is zero, Theorem~\ref{thm:all-one-bit} states that the expected runtime of the algorithm is at most $n^2$ iterations, and Theorem~\ref{thm:all-bitwise} gives a bound of $\frac{e^\chi n^2}{\chi} + O(n)$ iterations (in this case we can chose $c = 1$ to satisfy the assumptions of the theorem). These are larger than the bounds for the \oea with one-bit mutation and with standard bit mutation from~\cite{Rudolph97} and~\cite{BottcherDN10}, which indicates that our bounds are not tight by at least a constant factor of $\frac{1}{2}$. The main argument in those previous studies which is missing in our analysis is that all bits to the right of the left-most zero-bit in the current parent $x$ are distributed uniformly at random at any time. Therefore, each time the algorithm improves the fitness, the improvement is greater than $1$ in expectation (and most of the time it is close to $2$). It is not trivial to extend this argument to the noisy case, since situations when the true and noisy fitness of the current individual are not the same might lead to a non-uniform distributions of the bits in the suffix. Nevertheless, we are optimistic that modifying this argument will allow to improve our bounds in future research, and it will also help to prove lower bounds.

It is also interesting to note that when $q = o(1)$, the bounds from Theorems~\ref{thm:all-one-bit} and~\ref{thm:all-bitwise} stay the same as with $q = 0$, apart from a factor of $(1 + o(1))$. This suggests that the performance of the \oea without re-evaluations is not affected by even quite strong noise rates, as long as noise occurs only once in a super-constant number of evaluations on average. Without a matching lower bound on the noisy runtime we cannot formally state that this suggestion is true, but it is supported by the experiments in Section~\ref{sec:experiments} (see Figure~\ref{fig:runtimes}).

Finally, we note that the proofs of Theorems~\ref{thm:all-one-bit} and~\ref{thm:all-bitwise} are mostly different in how they treat the noise model, but the choice of the mutation operator is not so important. In all estimates of probabilities of progress (reducing the number of zero-bits in the active prefix of $x$ or having a successful phase) we only want the mutation to flip one particular bit and not to flip any other bits. This implies that the same arguments can be repeated for different combinations of mutation and noise, e.g., for one-bit noise and standard bit mutation or for bitwise noise and one-bit mutation.

\section{Experiments}
\label{sec:experiments}

In this section we present the results of our empirical study. We are interested in comparing the performance of the \oea with and without re-evaluations of the parent individual on noisy \leadingones. In particular, we want to see, at which problem sizes and at which noise rates the performance of these approaches starts to be different and which noise rates are tolerated by each approach.

We ran the \oea with and without re-evaluations on \leadingones with problem sizes $n \in \{2^3, 2^4, \dots, 2^9\}$ and with bitwise noise with rate $\frac{q}{n}$, where $q \in \{\frac{1}{n^2}, \frac{\ln(n)}{n^2}, \frac{1}{n}, 1\}$. As the mutation operator we used standard bit mutation with the most commonly used mutation rate $\frac{1}{n}$ (that is, $\chi = 1$). For each parameter setting (problem size $n$, mutation rate $q$ and with or without re-evaluation) we performed $128$ runs. Each run used a Python random generator which was initialized with a seed from $[0..127]$ (one run per each seed). Each run stopped after the algorithm found the optimum and evaluated its fitness correctly or after $100n^2$ iterations. The latter time limit was introduced, since in~\cite{Sudholt21} it was shown that the runtime of the \oea with re-evaluations is super-polynomial for $q=\omega(\frac{\log(n)}{n^2})$, hence running this algorithm with $q \in \{\frac{1}{n}, 1\}$ might be costly even for relatively small problem sizes.

For each run we logged if the run found the optimum and evaluated it correctly before the time limit, and if it did, then we logged the number of iterations it took. We also logged the best found true and noisy fitness values, which allows us to evaluate the performance of runs which did not find the optimum before the time limit.

Figure~\ref{fig:mean-fitness} shows the average best true fitness found by the algorithm for different parameter setting and different problem sizes. We normalize fitness by its maximum value (that is, by $n$) so that it was easier to compare the algorithms.

\begin{figure}
    \begin{center}
        \begin{tikzpicture}
            \begin{axis}[width=0.7\linewidth, height=0.3\textheight,
                cycle list name=tikzcycle, grid=major,  xmode=log, log base x=2, legend pos=outer north east, ymin = 0, ymax = 1.1,
                legend cell align={left},
                xlabel={Problem size $n$}, ylabel={Best fitness found $/n$}]

                \addplot plot [error bars/.cd, y dir=both, y explicit] coordinates
                    {(8,1.0)+-(0,0.0)(16,1.0)+-(0,0.0)(32,1.0)+-(0,0.0)(64,1.0)+-(0,0.0)(128,1.0)+-(0,0.0)(256,1.0)+-(0,0.0)(512,1.0)+-(0,0.0)};
                \addlegendentry{$q = \frac{1}{n^2}$};
                \addplot plot [error bars/.cd, y dir=both, y explicit] coordinates
                    {(8,1.0)+-(0,0.0)(16,1.0)+-(0,0.0)(32,1.0)+-(0,0.0)(64,1.0)+-(0,0.0)(128,1.0)+-(0,0.0)(256,1.0)+-(0,0.0)(512,1.0)+-(0,0.0)};
                \addlegendentry{$q = \frac{\ln(n)}{n^2}$};
                \addplot plot [error bars/.cd, y dir=both, y explicit] coordinates
                    {(8,1.0)+-(0,0.0)(16,1.0)+-(0,0.0)(32,1.0)+-(0,0.0)(64,1.0)+-(0,0.0)(128,1.0)+-(0,0.0)(256,1.0)+-(0,0.0)(512,1.0)+-(0,0.0)};
                \addlegendentry{$q = \frac{1}{n}$};
                \addplot plot [error bars/.cd, y dir=both, y explicit] coordinates
                    {(8,1.0)+-(0,0.0)(16,1.0)+-(0,0.0)(32,1.0)+-(0,0.0)(64,1.0)+-(0,0.0)(128,1.0)+-(0,0.0)(256,0.995574951171875)+-(0,0.05006371253518037)(512,1.0)+-(0,0.0)};
                \addlegendentry{$q = 1$};
                \addplot plot [error bars/.cd, y dir=both, y explicit] coordinates
                    {(8,1.0)+-(0,0.0)(16,1.0)+-(0,0.0)(32,1.0)+-(0,0.0)(64,1.0)+-(0,0.0)(128,1.0)+-(0,0.0)(256,1.0)+-(0,0.0)(512,1.0)+-(0,0.0)};
                \addplot plot [error bars/.cd, y dir=both, y explicit] coordinates
                    {(8,1.0)+-(0,0.0)(16,1.0)+-(0,0.0)(32,1.0)+-(0,0.0)(64,1.0)+-(0,0.0)(128,1.0)+-(0,0.0)(256,1.0)+-(0,0.0)(512,1.0)+-(0,0.0)};
                \addplot plot [error bars/.cd, y dir=both, y explicit] coordinates
                    {(8,1.0)+-(0,0.0)(16,1.0)+-(0,0.0)(32,1.0)+-(0,0.0)(64,1.0)+-(0,0.0)(128,0.91290283203125)+-(0,0.048359548851241156)(256,0.7318115234375)+-(0,0.038318572806443026)(512,0.5832366943359375)+-(0,0.02643459333672273)};
                \addplot plot [error bars/.cd, y dir=both, y explicit] coordinates
                    {(8,1.0)+-(0,0.0)(16,0.986328125)+-(0,0.030313222959963734)(32,0.68115234375)+-(0,0.06796475489435248)(64,0.4190673828125)+-(0,0.036810508138777134)(128,0.25732421875)+-(0,0.018851992531728274)(256,0.158599853515625)+-(0,0.009279417651146902)(512,0.0999755859375)+-(0,0.0056625509200086185)};
            \end{axis}
        \end{tikzpicture}
    \end{center}
    \caption{The normalized mean best true fitnesses and their standard deviation of the \oea with and without re-evaluations for different noise rates depending on the problem size. The results of the \oea without re-evaluations are shown with solid lines and filled markers and the runtimes of the \oea with re-evaluations are shown in dashed lines and unfilled markers.}
    \label{fig:mean-fitness}
\end{figure}
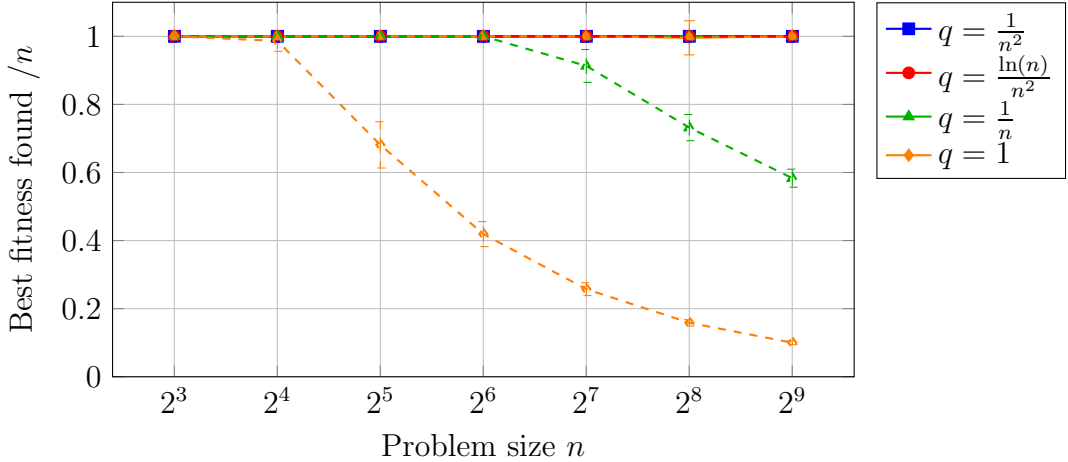

The results show that both algorithms are robust to small noise rates $q = \frac{1}{n^2}$ and $q = \frac{\ln(n)}{n^2}$ and find the optimal solution before the time limit. For noise rate $q = \frac{1}{n}$ the \oea with re-evaluations starts to struggle to find the optimum at problem size $2^7$. When the problem size grows to $2^9$, it never finds the optimum in time, but only reaches the fitness which is on average $0.59n$. 

For $q = 1$ the conditions of Theorem~\ref{thm:all-bitwise} are not satisfied for the considered $\chi = 1$, however the experiments show that the \oea without re-evaluations finds the optimum in almost all runs before the time limit. The only exception is the run with $n = 256$ and seed 8. The further study of this run showed that the algorithm at some point got a parent individual with three zeros in its active prefix. Lemma~\ref{lem:mut-and-noise} demonstrates that in this situation the probability to get to a better state with less zero-bits in the active prefix is of order $\frac{1}{n^3}$, hence it is natural that the algorithm in this situation does not find the optimum before the quadratic time limit. We ran the algorithm with the same seed without time limit, and it found the optimum in 9148621 iterations, which is approximately by factor $1.4$ larger than the time limit. The \oea with re-evaluations performed in this case very poorly. Even at problem size $2^4$ there were runs which did not find the optimum in time, and by problem size $2^9$ the average best fitness found by the algorithm in time limit was less than $0.10n$. 

For those parameters settings which allowed all 128 runs to finish before the time limit, we compare the average time it took the algorithm to find the optimum. The mean runtimes normalized by $n^2$ (which, up to a constant factor, is the asymptotic bound from Theorem~\ref{thm:all-bitwise}) are shown in Figure~\ref{fig:runtimes}. Note that for the run of the \oea without re-evaluations with $q = 1$, $n = 256$ and seed $8$ we used the runtime which was obtained when we performed this run without time limit.

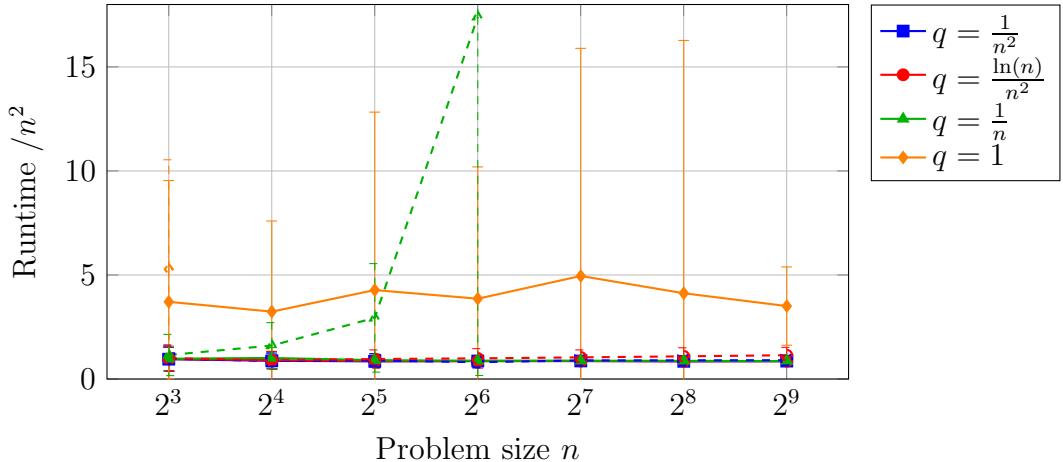
\begin{figure}
    \begin{center}
        \begin{tikzpicture}
            \begin{axis}[width=0.7\linewidth, height=0.3\textheight,
                cycle list name=tikzcycle, grid=major,  xmode=log, log base x=2, legend pos=outer north east, ymin = 0, ymax = 18,
                legend cell align={left},
                xlabel={Problem size $n$}, ylabel={Runtime $/n^2$}]

                \addplot plot [error bars/.cd, y dir=both, y explicit] coordinates
                    {(8,0.9530029296875)+-(0,0.5668456917487726)(16,0.875213623046875)+-(0,0.373704428182033)(32,0.8500900268554688)+-(0,0.267850233007382)(64,0.86724853515625)+-(0,0.20242606605082086)(128,0.8669672012329102)+-(0,0.15469472328556408)(256,0.8494489192962646)+-(0,0.09054699182851848)(512,0.8604410290718079)+-(0,0.07403481892098208)};
                \addlegendentry{$q = \frac{1}{n^2}$};
                \addplot plot [error bars/.cd, y dir=both, y explicit] coordinates
                    {(8,0.961181640625)+-(0,0.5643355889482661)(16,0.90411376953125)+-(0,0.42220690051443094)(32,0.8581161499023438)+-(0,0.2666612027924313)(64,0.8580741882324219)+-(0,0.20096153204737008)(128,0.8610916137695312)+-(0,0.13995495874313904)(256,0.8523719310760498)+-(0,0.09326239607095504)(512,0.8539990186691284)+-(0,0.06977669611536919)};
                \addlegendentry{$q = \frac{\ln(n)}{n^2}$};
                \addplot plot [error bars/.cd, y dir=both, y explicit] coordinates
                    {(8,0.9754638671875)+-(0,0.6045392355268789)(16,1.011138916015625)+-(0,0.5009974560964776)(32,0.9066238403320312)+-(0,0.29684897408307126)(64,0.87799072265625)+-(0,0.21285815832905647)(128,0.8793082237243652)+-(0,0.15012400282130942)(256,0.8735600709915161)+-(0,0.09543288974620347)(512,0.8613344728946686)+-(0,0.07571242776004798)};
                \addlegendentry{$q = \frac{1}{n}$};
                \addplot plot [error bars/.cd, y dir=both, y explicit] coordinates
                    {(8,3.7109375)+-(0,5.827214066943842)(16,3.238555908203125)+-(0,4.358411289623043)(32,4.275550842285156)+-(0,8.552553154992491)(64,3.860382080078125)+-(0,6.334360383784051)(128,4.9519453048706055)+-(0,10.941329564123137)(256,4.1229212284088135)+-(0,12.151007797341176)(512,3.5059000849723816)+-(0,1.8823850010860306)};
                \addlegendentry{$q = 1$};
                \addplot plot [error bars/.cd, y dir=both, y explicit] coordinates
                    {(8,0.978759765625)+-(0,0.6105575870293305)(16,0.86712646484375)+-(0,0.3822480800291252)(32,0.8916015625)+-(0,0.32650157275452824)(64,0.8085975646972656)+-(0,0.17284887981144556)(128,0.8935604095458984)+-(0,0.17407508114842382)(256,0.8975763320922852)+-(0,0.17531449202247154)(512,0.9043707251548767)+-(0,0.1604770417098358)};
                \addplot plot [error bars/.cd, y dir=both, y explicit] coordinates
                    {(8,1.0029296875)+-(0,0.6395656030600911)(16,0.905548095703125)+-(0,0.4216335481202155)(32,0.9594802856445312)+-(0,0.4413968591553022)(64,0.99267578125)+-(0,0.4650022304330588)(128,1.0398306846618652)+-(0,0.36075445247833504)(256,1.0899869203567505)+-(0,0.41104752037734843)(512,1.1427583396434784)+-(0,0.38043795393282187)};
                \addplot plot [error bars/.cd, y dir=both, y explicit] coordinates
                    {(8,1.1531982421875)+-(0,0.9907511256338397)(16,1.606903076171875)+-(0,1.1026229040610003)(32,2.9396743774414062)+-(0,2.6074291446043)(64,17.410837173461914)+-(0,17.23971716959657)};
                \addplot plot [error bars/.cd, y dir=both, y explicit] coordinates
                    {(8,5.2772216796875)+-(0,5.260056704588453)};
            \end{axis}
        \end{tikzpicture}
    \end{center}
    \caption{The normalized mean runtimes and their standard deviation of the \oea with and without re-evaluations for different noise rates depending on the problem size. The runtimes of the \oea without re-evaluations are shown with solid lines and filled markers and the runtimes of the \oea with re-evaluations are shown in dashed lines and unfilled markers.}
    \label{fig:runtimes}
\end{figure}

In this plot we see that when the noise rate is small, namely, $q = \frac{1}{n^2}$, the runtime of both algorithms is very close to the runtime of the \oea in the noiseless case, which, as it was shown in~\cite{BottcherDN10}, is approximately $0.86n^2$. For the larger noise rates $q = \frac{\ln(n)}{n^2}$ and $q = \frac{1}{n}$ the runtime of the \oea without re-evaluations stays the same, close to $0.86n^2$. For $q = \frac{\ln(n)}{n^2}$ the runtime of the \oea with re-evaluations grows a little faster than with $q = \frac{1}{n^2}$. Recalling the theoretical bound from~\cite{Sudholt21} (see Section~\ref{sec:oea}), we can conclude that the leading constant in the exponential term of that bound is relatively small, and the runtime grows just slightly faster than $\Theta(n^2)$ in this case. Noise rate $q = \frac{1}{n}$ yields a significantly faster growth of runtime of the \oea with re-evaluations, which can be observed already at small problem sizes.

Noise rate $q = 1$ (which, we recall, is too strong to satisfy conditions of Theorem~\ref{thm:all-bitwise}) is harder to deal with for both versions of the algorithm. The \oea with re-evaluations finished all 128 runs before the time limit only for problem size $n = 2^3$, hence there is only one point in the plot. The runtime of the \oea without re-evaluations in this case is significantly larger than with smaller mutation rates, and it is much less concentrated. It seems to grow not faster than $n^2$ (up to a constant factor), but the high variance of the runtimes does not allow to conclude this. 

We also note that Figure~\ref{fig:runtimes} shows the runtimes measured as the number of iterations for both algorithms. If we measure it as the number of fitness evaluations, then the runtime of the \oea with re-evaluations will be twice as large, and therefore, its performance is inferior to the \oea without re-evaluations for all noise rates.

\subsection{Statistical Tests}

In this section we show the results of statistical tests for the values shown in Figures~\ref{fig:mean-fitness} and~\ref{fig:runtimes} for the maximum problem size $n = 2^9$.

For Figure~\ref{fig:mean-fitness} it does not make sense to compare the best fitness of the settings, in which the global optimum was always found. Hence, we compare the best found fitness of the \oea with re-evaluations for $q = \frac{1}{n}$ and $q = 1$ and all other settings, for which the best found fitness is always 512. We perform Welch's t-test, which is a modification of Student's t-test for samples with different variance. The obtained p-values are shown in Table~\ref{tbl:best-fitness}. As the test shows, the considered samples are significantly different.

\begin{table*}
    \centering
    \begin{tabular}{l|lll}
        \toprule
        & $q = \frac{1}{n}$ & $q = 1$ & All others \\[5pt] \hline
        & & & \\[-8pt]
        $q = \frac{1}{n}$ & $1.0$ & $3.376 \cdot 10^{-173}$ & $2.593 \cdot 10^{-154}$ \\[5pt]
        $q = 1$ & $3.376 \cdot 10^{-173}$ & $1.0$ & $1.186 \cdot 10^{-281}$ \\[5pt]
        All others & $2.593 \cdot 10^{-154}$ & $1.186 \cdot 10^{-281}$ & --- \\[5pt]
        \bottomrule
    \end{tabular}
    \caption{The p-values of Welch's t-test performed on the samples of the best fitness for different parameter settings of the \oea and noise rates. The first two columns and the first two rows represent the runs of the \oea with re-evaluations.}
    \label{tbl:best-fitness}
\end{table*}

For Figure~\ref{fig:runtimes} we compare all parameters which are present in that figure, that is, the runtimes of the \oea without re-evaluations with all noise rates and of the \oea with re-evaluations with $q = \frac{1}{n^2}$ and $q = \frac{\ln(n)}{n^2}$. We also performed Welch's t-test, since there is no guarantee that the variances are the same. The resulting p-values (for a two-sided null hypothesis) are shown in Table~\ref{tbl:runtimes}. In that table we see that the runtimes of the \oea without re-evaluations do not significantly differ for all $q \le \frac{1}{n}$, while all other pairs of settings have a significant difference in ther mean runtimes.

\begin{landscape}
    \begin{table}
        \centering
        \begin{tabular}{l|llllll}
            \toprule
            & $q = \frac{1}{n^2}$ & $q = \frac{1}{n^2}$, Re & $q = \frac{\ln(n)}{n^2}$ & $q = \frac{\ln(n)}{n^2}$, Re & $q = \frac{1}{n}$ & $q = 1$\\[5pt] \hline
            & & & & & \\[-8pt]
            $q = \frac{1}{n^2}$ &$1.0$ & $5.470 \cdot 10^{-3}$ & $0.474$ & $1.245 \cdot 10^{-13}$ & $0.924$ & $5.205 \cdot 10^{-32}$ \\[5pt]
            $q = \frac{1}{n^2}$, Re & $5.470 \cdot 10^{-3}$ & $1.0$ & $1.356 \cdot 10^{-3}$ & $7.158 \cdot 10^{-10}$ & $6.681 \cdot 10^{-3}$ & $2.016 \cdot 10^{-31}$ \\[5pt]
            $q = \frac{\ln(n)}{n^2}$ &$0.474$ & $1.356 \cdot 10^{-3}$ & $1.0$ & $4.114 \cdot 10^{-14}$ & $0.421$ & $4.241 \cdot 10^{-32}$ \\[5pt]
            $q = \frac{\ln(n)}{n^2}$, Re & $1.245 \cdot 10^{-13}$ & $7.158 \cdot 10^{-10}$ & $4.114 \cdot 10^{-14}$ & $1.0$ & $1.472 \cdot 10^{-13}$ & $4.980 \cdot 10^{-28}$ \\[5pt]
            $q = \frac{1}{n}$ &$0.924$ & $6.681 \cdot 10^{-3}$ & $0.421$ & $1.472 \cdot 10^{-13}$ & $1.0$ & $5.352 \cdot 10^{-32}$ \\[5pt]
            $q = 1$ &$5.205 \cdot 10^{-32}$ & $2.016 \cdot 10^{-31}$ & $4.241 \cdot 10^{-32}$ & $4.980 \cdot 10^{-28}$ & $5.352 \cdot 10^{-32}$ & $1.0$ \\[5pt]
            \bottomrule
        \end{tabular}
        \caption{The p-values of Welch's t-test performed on the samples of the runtimes for different parameter settings of the \oea and noise rates. ``Re'' in the top row and in the left column indicates that this parameter setting is for the \oea with re-evaluations.}
        \label{tbl:runtimes}
    \end{table}
\end{landscape}

\section{Conclusion}

In this paper we have shown that simple random search heuristics can be very robust to noise, even if a constant fraction of all fitness evaluations are faulty due to the noise. For this, however, it is necessary that the algorithm does not re-evaluate solutions, as if it was unaware of optimizing a noisy function. We showed that this approach leads to a much better performance than the approach dominant in theoretical works, where parent individuals are re-evaluated in each iteration.

From our mathematical analysis we can also see the reason for this different behavior. Without re-evaluations, if we accept an individual based on an incorrectly evaluated fitness, then an individual which truly has this fitness is not far away, and the algorithm is capable to fix its mistake by generating that individual. Hence the negative impact of having an individual with incorrect fitness in the population, or the positive effect of re-evaluations, is not as big as previously thought. On the other hand, re-evaluations allow for a second type of mistake, namely wrongly evaluating the parent so that it appears worse than it is, and then easily accepting a worse offspring. Since generating worse offspring is relatively easy when close to the optimum, this type of mistake appears relatively often with high noise rate. 
We note that this explanation is not specific to \leadingones benchmark. For this reason, we are optimistic that our findings generalize to other optimization problems, and this is surely an interesting direction for future research.

% \bibliographystyle{alpha}
% \bibliography{alles_ea_master,ich_master,rest}

\newcommand{\etalchar}[1]{$^{#1}$}

}%end sloppy

\end{document}